%% file: paper.tex
\pdfoutput=1

\documentclass[11pt,a4paper]{article}
\usepackage[hidelinks]{hyperref}

\input{math.tex}
\usepackage{style}
\usepackage{url}

\usepackage{gb4e}
\noautomath

\usepackage[acceptedWithA]{tacl2021v1}

\usepackage{microtype}

\usepackage{inconsolata}

\usepackage{centernot}
\usepackage{wasysym}
\usepackage{marvosym}

\usepackage{forest}

\usepackage{adjustbox}
\usepackage{dashbox}
\newcommand{\coloredbox}[2]{%
    \colorlet{currentcolor}{.}%
    {\color{#1}%
    \fbox{\color{currentcolor}#2}}%
}
\newcommand{\coloreddashedbox}[2]{%
    \colorlet{currentcolor}{.}%
    {\color{#1}%
    \adjustbox{precode=\dbox}{\color{currentcolor}#2}}%
}
\definecolor{boxbg}{rgb}{0.85,0.85,0.85}
\newcommand{\shadedbox}[1]{%
    {\fcolorbox{white}{boxbg}{\begin{varwidth}{\textwidth}\centering #1\end{varwidth}}}%
}

\usepackage{xspace,mfirstuc,tabulary}

\newif\iftaclinstructions
\taclinstructionsfalse %
\iftaclinstructions
\renewcommand{\confidential}{}
\renewcommand{\anonsubtext}{(No author info supplied here, for consistency with
TACL-submission anonymization requirements)}
\newcommand{\instr}
\fi

\newcommand{\dn}[1]{\sv{#1}}
\newcommand{\dnc}[2]{\dn{#1 | #2}}
\newcommand{\dnl}[1]{\sv{#1}_L}
\newcommand{\dnlc}[2]{\dnl{#1 | #2}}

\newcommand{\dntc}[2]{\sv{\text{#1} | #2}}

\newcommand{\assert}{\aleph}
\newcommand{\attnprobe}{\textsc{+Attn}\xspace}
\newcommand{\noattnprobe}{\textsc{--Attn}\xspace}

\definecolor{orange}{rgb}{1,0.5,0}
\definecolor{mdred}{rgb}{0.7,0,0}
\definecolor{mdgreen}{rgb}{0.05,0.6,0.05}
\definecolor{mdblue}{rgb}{0,0,0.7}
\definecolor{dkblue}{rgb}{0,0,0.5}
\definecolor{dkgray}{rgb}{0.3,0.3,0.3}
\definecolor{slate}{rgb}{0.25,0.25,0.4}
\definecolor{gray}{rgb}{0.5,0.5,0.5}
\definecolor{ltgray}{rgb}{0.7,0.7,0.7}
\definecolor{purple}{rgb}{0.7,0,1.0}
\definecolor{lavender}{rgb}{0.65,0.55,1.0}

\newcommand{\papercomment}[3]{\ensuretext{\textcolor{#3}{[#1 #2]}}}
\renewcommand{\papercomment}[3]{}  %

\definecolor{revcolor}{rgb}{0.04,0.6,0.04}
\newcommand{\rev}[1]{\textcolor{black}{#1}}
\newcommand{\revv}[1]{\textcolor{black}{#1}}

\title{Transparency Helps Reveal When Language Models Learn Meaning}

\author{Zhaofeng Wu$^\text{\Cancer}$ \quad
    William Merrill$^\text{\Scorpio}$ \quad
    Hao Peng$^\text{\Leo}$ \quad
    Iz Beltagy$^\text{\Leo}$ \quad
    Noah A. Smith$^{\text{\Leo\ } \rotatebox[y=0.1cm]{90}{\textsuperscript{\Cancer}}}$ \\
    $^\text{\Cancer}$MIT \quad $^\text{\Scorpio}$New York University \quad $^\text{\Leo}$Allen Institute for Artificial Intelligence \\
    $^{\rotatebox[y=0.1cm]{90}{\textsuperscript{\Cancer}}}$\hspace{-0.1cm}Paul G. Allen School of Computer Science \& Engineering, University of Washington \\
    \texttt{zfw@csail.mit.edu \quad willm@nyu.edu \quad \{haop,beltagy,noah\}@allenai.org}
}

\begin{document}
\maketitle
\begin{abstract}

Many current NLP systems are built from language models trained to optimize unsupervised objectives on large amounts of raw text. Under what conditions might such a procedure acquire meaning? Our systematic experiments with synthetic data reveal that, with languages where all expressions have context-\emph{independent} denotations (i.e., languages with \term{strong transparency}), both autoregressive and masked language models successfully learn to emulate semantic relations between expressions. However, when denotations are changed to be context-\emph{dependent} with the language otherwise unmodified, this ability degrades. Turning to natural language, our experiments with a specific phenomenon---referential opacity---add to the growing body of evidence that current language models \revv{do not represent natural language semantics well}. We show this failure relates to the context-dependent nature of natural language form-meaning mappings.
\iftaclpubformat
\blfootnote{This work was done when Zhaofeng Wu was at AI2.}
\blfootnote{Our code and trained models are released at \url{https://github.com/ZhaofengWu/transparency}.}
\else
\blfootnote{Our code and trained models will be released at \url{https://github.com/anonymized_link}.}
\fi

\end{abstract}

\input{sections/01_intro}
\input{sections/02_background}

\input{sections/03_lm}

\input{sections/04_nl}

\input{sections/05_discussion}
\input{sections/06_i_dont_want_to_have_a_related_work_section_but_the_reviewer_wont_be_happy_without_it_so_i_have_to_yield}
\input{sections/07_now_that_we_have_a_related_work_section_we_have_to_have_a_conclusion_section_too_otherwise_it_would_be_unconventional}

\iftaclpubformat
\section*{Acknowledgments}
We thank the TACL reviewers and action editor for helpful feedback on this work.  We thank Kyle Richardson, Jesse Dodge, and other members of AI2 for insightful discussions.
This work was funded in part by NSF award 1922658. WM was supported by an NSF graduate research fellowship.
\fi

\bibliography{custom}
\bibliographystyle{acl_natbib}

\clearpage
\appendix

\input{sections/99_appendix}

\end{document}

%% file: math.tex
\usepackage{amsmath,amsfonts,amsthm,amssymb,bm}

\def\eqref#1{equation~\ref{#1}}

\def\1{\bm{1}}

\DeclareMathAlphabet{\mathsfit}{\encodingdefault}{\sfdefault}{m}{sl}
\SetMathAlphabet{\mathsfit}{bold}{\encodingdefault}{\sfdefault}{bx}{n}

\newcommand{\dpt}[1]{\left\langle#1\right\rangle} %
\newcommand{\sv}[1]{\ensuremath{\llbracket #1 \rrbracket}}

\let\emptyset\varnothing

\newtheorem{theorem}{Theorem}

\theoremstyle{definition}
\newtheorem{definition}{Definition}

\theoremstyle{definition}
\newtheorem{example}{Example}

%% file: sections/01_intro.tex
\section{Introduction}

Despite language models' (LMs) centrality to recent progress on NLP benchmarks, a formal characterization of what can be learned from unsupervised training on large text corpora, and of what modern 
language models actually do learn, remains elusive.  
Empirically, \citet{tenney2018what}, \citet{kovaleva-etal-2019-revealing}, \citet{wu-etal-2021-infusing}, \emph{i.a.}, all discovered that pretrained LMs possess unsatisfactory semantic representations.
\citet{traylor-etal-2021-mean} found co-variation between form and meaning to be insufficient for an LM to represent lexical semantics.
\citet{li-etal-2021-implicit}, on the other hand, identified evidence of LMs representing dynamic semantics~\citep{Kamp2011,Heim2012,Groenendijk1991}.

From first principles, \citet{bender-koller-2020-climbing} argued that it is \emph{a priori} impossible for an ungrounded system that has access only to linguistic forms to learn the mapping between those forms and their grounded denotations. 
They claimed, as a thought experiment, that a learner that has access to all Java code (i.e., form) on GitHub can never learn execution (i.e., meaning). They nevertheless acknowledged that the existence of unit tests, which assert the expected output given input to blocks of code, could constitute a weak form of grounding which potentially enables the learning of meaning.

Formalizing this idea, \citet{merrill-etal-2021-provable} theoretically proved the possibility of learning (or more technically, emulating) semantic relations between expressions in a certain class of formal languages---those that are \term{strongly transparent} whose expressions have context-independent denotations---using an assertion oracle, analogous to the assertions in unit tests.
In addition, with an example, they showed the existence of non-emulatable languages even with an assertion oracle. 

Yet, the practical implications of these theoretical results have not been explored.
While assertions enable the emulation of strongly transparent languages, it is unclear if existing LM architectures and objectives \emph{achieve} emulation given training data with assertions.
Furthermore, we do not know if natural language (NL) is similarly non-emulatable as \citet{merrill-etal-2021-provable}'s constructed example, especially since non-transparency does not always imply non-emulatability.
We thus pose two research questions:
\begin{compactitem}
\item[\textbf{RQ1.}] 
Can current LM architectures and pretraining objectives emulate the meaning of strongly transparent languages?
\item[\textbf{RQ2.}] Can modern LMs fully emulate the meaning of natural language which is non-transparent?
\end{compactitem}

We answer RQ1 in the positive (\S\ref{sec:lm}):
on a strongly transparent propositional logic language,
autoregressive and masked language models pretrained on only expressions (form), à la GPT-2~\citep{gpt2} and RoBERTa~\citep{roberta}, can
consistently compare and evaluate their values (meaning).
\rev{We find that necessary grounding of the pretraining data distribution is crucial to this ability.}
We also investigate the role of transparency for emulatability in a controlled setting as an intermediate study before analyzing non-transparent natural language.
We ablate strong transparency from the logic language while keeping other factors unchanged.
We observe a substantial drop in the LMs' ability to emulate meaning, highlighting the importance of transparency for emulatability.

We then turn to natural language (\S\ref{sec:nl}). Referential opacity is an extensively studied phenomenon in semantics and philosophy~\citepia{quine1956quantifiers,kripke1972naming} but has not been examined in modern NLP. We prove that this phenomenon entails non-transparency
and analyze how well existing LMs represent it.
Our analyses based on probing and sentence similarity
point to a lack of its representation in the largest GPT-2 and BERT~\citep{bert} models (RQ2).
Theoretically, this is a natural language parallel to the emulation difficulty for our non-transparent formal language, and further reinforces the connection between transparency and meaning emulatability.
Practically, through the lens of strong transparency, our results supplement prior studies that identified pretrained LMs' insufficient semantic representations~\citepia{tenney2018what,yu-ettinger-2020-assessing,yu-ettinger-2021-interplay,wu-etal-2021-infusing}.

%% file: sections/02_background.tex
\section{Background}

We follow \citet{merrill-etal-2021-provable}'s operationalization of the learning of meaning by emulation and their definition of strong transparency. We summarize their nomenclature and theoretical results in this section and provide some examples. We refer readers to \citet{merrill-etal-2021-provable} for more details.

At a high level, we take \rev{an \emph{inferential} \citep[\S2.2.3]{sep-meaning}} view of meaning.
An LM is taken to understand a language $L$ if it can resolve semantic relations (e.g., equivalence) between expressions in $L$.\footnote{This \rev{inferentialist} perspective \rev{can be contrasted with} \emph{denotationalism}, \rev{which says that ``understanding'' is the task of mapping an expression to a logical representation of its meaning} \citep[\S2.2.3]{sep-meaning}. Inferentialism implicitly underlies
natural language inference-based evaluation of NLP models \citep[e.g.,][]{bowman-etal-2015-large}.} This is achieved through two procedures: $\mu_L$ maps expressions into representations based on training data from $L$, and $\delta$ uses the representations of two expressions to resolve a semantic relation between them.

\subsection{Languages}

We consider a \term{language} $L \subseteq \Sigma^*$ over an alphabet $\Sigma$ and denote $\left(\Sigma^*\right)^2=\Sigma^* \times \Sigma^*$. 
We term members of $L$ \term{sentences}. 
We consider an \term{expression} $e\in \Sigma^*$ with associated left and right \term{context} $\kappa=\dpt{l, r}\subseteq\left(\Sigma^*\right)^2$. $ler\in L$ is a sentence.
We denote the empty string with $\lambda$ and the empty context with $\lambda^2$.

\begin{definition}[$L_t$]\label{ex:logical}
We use the following context-free grammar (CFG) to specify a propositional logic language as a running example:
\begin{align} \label{eq:cfg}
\begin{split}
    S &\rightarrow (e \land e) \ | \ (e \lor e) \ | \ (\lnot e) \\
    e &\rightarrow (e \land e) \ | \ (e \lor e) \ | \ (\lnot e) \ | \ \texttt{T} \ | \ \texttt{F}
\end{split}
\end{align}
\revv{$S$ is the distinguished start symbol and} \texttt{T} and \texttt{F} stand for True and False.
We call this language $L_t$ where $t$ stands for ``transparent'' (see \S\ref{sec:strong-transparency}).
It underlies our investigation in \S\ref{sec:lm}.
\end{definition}

\revv{For example, the sentence $(((\lnot \texttt{T}) \lor \texttt{F})\lor(\lnot \texttt{T}))$ belongs to $L_t$ because it can be generated by this CFG using the steps illustrated in Figure~\ref{fig:logic}. In this sentence, the expression $\texttt{F}$ has context $\dpt{\ (((\lnot \texttt{T}) \lor \ \ , \ \ )\lor(\lnot \texttt{T}))\ }$.}

\subsection{Meaning} \label{sec:meaning}

We consider the denotation of an expression $e$, $\dnlc{e}{\kappa}$, to be its meaning in the context $\kappa$.\footnote{
We overload $L$ to represent both the surface form and a mapping between form and denotation.} We write $\dnlc{e}{\kappa}=\emptyset$ if $e$ is invalid in $\kappa$.

The meaning of a propositional logic expression can be the value derived from its conventional semantics, i.e., either \texttt{T} or \texttt{F}. 
\rev{For instance, $\dnc{(\texttt{T}\land(\lnot\texttt{F}))}{\lambda^2}_{L_t}=\texttt{T}$, and $\dnc{(\lnot\texttt{F})}{\dpt{(\texttt{T}\land \ , \ )}}_{L_t}=\texttt{T}$.}
For natural language, extensionally, the meaning of a sentence is its truth value, also either \texttt{T} or \texttt{F}~\citep{Frege1892}; intensionally, the meaning is its truth condition, which could be viewed as a set of possible worlds where the sentence is true~\citep{Carnap1947}.
For a summary of the extension and intension of other expressions in NL, see \citet[\S1.3]{kearns2011semantics}.
\rev{As an example in English, extensionally, $\llbracket $An author of this paper believes that Corgis are the cutest dogs.$ | \lambda^2 \rrbracket=\texttt{T}$.}

\subsection{Assertion Oracle} \label{sec:assert}

To represent assertions in unit tests, \citet{merrill-etal-2021-provable} considered an assertion oracle which outputs if two expressions have the same denotation under the same context.
Specifically, for expressions $e, e' \in \Sigma^*$ and $\kappa \in \left(\Sigma^*\right)^2$, the assertion oracle is defined as
\begin{align} \label{eq:assert}
\assert_L\left (e,e' \mid \kappa\right) = \begin{cases}
1 & \text{if } \dnlc{e}{\kappa} = \dnlc{e'}{\kappa} \\
0 & \text{otherwise}
\end{cases}
\end{align}

LM pretraining corpora could provide $\assert$-like signals. %
For instance, pretraining sequences of the form \texttt{e=e'} are a natural analog to an $\assert$ query. We adopt this view to pretrain our propositional logic language in \S\ref{sec:lm}. In English and many other natural languages, copulas are a straightforward counterpart: ``Corgis are the cutest dogs.'' is equivalent to ``Corgis\texttt{=}the cutest dogs.'' This can be further reduced to all propositions: ``Corgis run.'' is equivalent to $\assert$(Corgis run., \texttt{T}) under the extensional framework.\footnote{Assuming that propositions are more frequently true than false, which tends to be the case pragmatically~\citep{grice1975logic}.}

\subsection{$\assert$-emulation: Learning Meaning} \label{sec:emulation}

\citet{merrill-etal-2021-provable} say that a class of languages $\mathcal L$ is $\aleph$-emulatable if, intuitively, a learner $\mu_L$ with $\assert_L$-access produces context-independent representations that allow another function $\delta$ to check the equivalence of any two expressions under any context without further $\assert_L$-access.
Formally, $\mathcal L$ is $\assert$-emulatable if there exists an oracle Turing machine $\mu_L$ (that can query $\aleph_L$) and a standard Turing machine $\delta$ such that, for all $L \in \mathcal L$, context $\kappa \in \left(\Sigma^*\right)^2$, and valid expressions $e, e'$ in $\kappa$,
\begin{align} \label{eq:emulation}
    \hspace{-0.1cm}\dnlc{e}{\kappa} = \dnlc{e'}{\kappa} \Longleftrightarrow \delta\left(\mu_L(e), \mu_L(e') \mid \kappa\right)
\end{align}

\rev{Back to Corgis, an English learner $\mu$ can observe the equivalence of $e=$ ``Corgis'' and $e'=$ ``the cutest dogs'' in many different contexts $\kappa$ and develop their representations. We say that natural language is emulated if there exists $\delta$ that can decide the equivalence between such expressions from the representations alone.}

The standard pretraining-probing setup is an intuitive instantiation of $\mu_L$ and $\delta$. A model $\mu_L$ can query $\assert_L$ while pretraining on language $L$, which can then produce a representation $\mu_L(e)$ for any expression $e$. An equivalence probe $\delta$ can take the (frozen) representation of two expressions and decide their equivalence in some context. \rev{Importantly, because $\delta$ is frozen, it cannot make any more queries to $\assert_L$.} We adopt this paradigm for analysis in \S\ref{sec:lm} and \S\ref{sec:nl} and elaborate below.

\subsection{Strong Transparency} \label{sec:strong-transparency}

\begin{definition} \label{def:strong-transparency}
A language $L$ is \term{strongly transparent} if all of its expressions have context-independent denotations. That is, for all $e\in\Sigma^*, \kappa \in \left(\Sigma^*\right)^2$, either $\dnlc{e}{\kappa} = \dnlc{e}{\lambda^2} \ne \emptyset$ or $\dnlc{e}{\kappa} = \emptyset$.
\end{definition}

Under conventional propositional logic semantics, $L_t$ (Def.~\ref{ex:logical}) is strongly transparent because the value of every expression is determined by itself and unaffected by its context. Natural language, on the other hand, is non-transparent. We prove in \S\ref{sec:nl} that the NL phenomenon of referential opacity violates strong transparency.

\citet{merrill-etal-2021-provable} theoretically proved that all strongly transparent languages are $\assert$-emulatable. \rev{In other words, it is possible to learn to emulate the meaning of these languages with only assertion oracle access.} The converse is not necessarily true\footnote{Consider, for example, a finite non-transparent language whose denotation space can be learned by enumeration.} and hence there may be a weaker condition than strong transparency that also entails $\assert$-emulatability. %

In what follows, we study how their theoretical results realize empirically. We examine in \S\ref{sec:lm} if LM architectures and objectives can emulate the meaning of a strongly transparent language.
In \S\ref{sec:nl}, we return to natural language which is non-transparent and thus \citet{merrill-etal-2021-provable}'s results do not predict its meaning emulatability.

%% file: sections/03_lm.tex
\section{How Well Do Language Models Fare?} \label{sec:lm}

While strongly transparent languages are in theory $\assert$-emulatable, it is unknown if existing LM architectures, coupled with their pretraining objectives, 
are able to successfully achieve $\assert$-emulation, or more intuitively, to learn their meaning.

To test this, we synthetically create a strongly transparent language based on propositional logic.
We pretrain LMs with the same architecture and similar data scale as GPT-2 and RoBERTa on a generated pretraining corpus.
We then train an equivalence probe
to study if the pretrained representations enable $\assert$-emulation. The probe is trained with a sentence pair binary classification objective and tested on unseen sentences sampled from the same grammar.
Alternatively, we also try to directly evaluate the value of unseen sentences, without probe training.
To isolate the effect of strong transparency, we also minimally perturb this language to be non-transparent and study how this affects emulatability.

\subsection{Data} \label{sec:datasets}
We use a PCFG to construct our propositional logic dataset because its recursive nature and context-freeness bear some resemblance to natural language,\footnote{\rev{There are aspects of natural language that a PCFG does not capture, such as recursion constraints~\citep{Karlsson+2010+43+68} and non-context-free phenomena~\citep{Shieber1985EvidenceAT}. Nevertheless, the goal of this research question is not to maximally simulate NL,
but rather investigate the distributional learnability of compositional semantics.
Future work could investigate the effect of moving away from a strict PCFG.
}} \rev{and because it is convenient for sampling}. The rules are specified in Eq.~\ref{eq:cfg} and the probabilities are hand-designed.
The denotation of an expression can be computed according to the conventional semantics of propositional logic,
which, as argued in \S\ref{sec:strong-transparency}, makes $L_t$ transparent.
Figure~\ref{fig:logic} shows an example.
See \S\ref{sec:details-pl} for more details.

Our CFG rules prevent the atomic sentences \texttt{T} and \texttt{F} from occurring in the corpus (and $($\texttt{T}$)$ and $($\texttt{F}$)$ too)
and only allow compositional sentences.
This ensures the absence of pretraining sequences like \texttt{sentence=T} and guarantees that there is no direct grounding to denotations during pretraining, but only indirect grounding via $\assert$. This makes the task more difficult than the $\assert$-emulation setup but more realistically transferable to natural language~(\S\ref{sec:discussion}).

\begin{figure}[t!]
    \begin{forest}
        for tree={fit=band,inner sep=1.8pt,l=6pt,s sep=2pt, align=center},
        [$S$\vspace{-3pt}\\\vspace{-3pt}\shadedbox{{\footnotesize $L_t$ \texttt{F}}\\{\footnotesize $L_n$ \texttt{T}}}
            [(, for tree={s sep=0pt}]
            [$e$\\\vspace{-3pt}\shadedbox{{\footnotesize $L_t$ \texttt{F}}\\{\footnotesize $L_n$ \texttt{T}}} [(] [$e$\vspace{-3pt}\\\vspace{-3pt}\shadedbox{{\footnotesize $L_t$ \texttt{F}}\\{\footnotesize $L_n$ \texttt{T}}} [(] [$\lnot$] [$e$\vspace{-3pt}\\\vspace{-3pt}\shadedbox{{\footnotesize $L_t$ \texttt{T}}\\{\footnotesize $L_n$ \texttt{F}}} [\coloredbox{red}{\texttt{T}}] ] [)] ] [$\lor$] [$e$\vspace{-3pt}\\\vspace{-3pt}\shadedbox{{\footnotesize $L_t$ \texttt{F}}\\{\footnotesize $L_n$ \texttt{F}}} [\texttt{F}, name=n1] ] [)] ]
            [$\lor$, for tree={s sep=0pt}]
            [\coloreddashedbox{blue}{$e$}\\\vspace{-3pt}\shadedbox{{\footnotesize $L_t$ \texttt{F}}\\{\footnotesize $L_n$ \texttt{F}}} [(] [$\lnot$] [$e$\vspace{-3pt}\\\vspace{-3pt}\shadedbox{{\footnotesize $L_t$ \texttt{T}}\\{\footnotesize $L_n$ \texttt{T}}} [\texttt{T}, name=n2] ] [)] ]
            [), for tree={s sep=0pt}]
        ]
    \end{forest}
	\caption{An example sentence in our propositional logic language as specified in Eq.~\ref{eq:cfg}. The \coloreddashedbox{blue}{$e$} node c-commands the \coloredbox{red}{\texttt{T}} node, inverting its meaning in $L_n$~\revv{(\S\ref{sec:non-transparency})}. We \shadedbox{mark} the denotation of each node under $L_t$ or $L_n$.
	}
    \vspace{-10pt}
	\label{fig:logic}
\end{figure}

The dataset has 819.2M pretraining sequences and 1M/10K/10K probe training/validation/test sentence pairs. All splits have disjoint sentences. The average sentence length is around 48.6. \S\ref{sec:details-pl} contains more details including tokenization.

\subsection{Pretraining} \label{sec:pretraining}
We pretrain from scratch an autoregressive LM (\textbf{ALM}) and a masked LM (\textbf{MLM}), respectively simulating GPT-2-small and RoBERTa-base\footnote{We do not follow BERT because next sentence prediction is not applicable here, but they are otherwise similar. \label{fn:no-bert}
} with their original architecture, objective, and, to the extent possible, hyperparameters.
They have near-identical model size hyperparameters, leading to 86.8M ALM parameters and 87.0M for MLM.
We sample sentence pairs (\texttt{a}, \texttt{b}) with the same denotation and
format the pretraining sequences in the form of \texttt{a=b}, such as \texttt{$($T$\land$F$)$=$($F$\lor$F$)$}, simulating $\assert$-access (but restricting queries to be sentences, a more challenging setup: see Eq.~\ref{eq:assert}).
\S\ref{sec:lm-probing} will discuss a necessary form of data augmentation.
We train for 100K steps, 20\% of RoBERTa-base's training duration and hence data size, which we found sufficient for convergence on our data.
\S\ref{sec:training-details-pl} summarizes hyperparameters.

\subsection{Analysis: Probing $L_t$} \label{sec:lm-probing}

\rev{Probing is a commonly adopted method to quantify the extent to which a representation encodes a particular type of linguistic information~\citepia{alain2017understanding,liu-etal-2019-linguistic,hewitt-manning-2019-structural}.
The representation is frozen, on top of which a lightweight classifier is trained to predict the information of interest.
As shown in \S\ref{sec:emulation}, this paradigm conveniently corresponds to the formalization in \citet{merrill-etal-2021-provable}, and hence we use it to investigate whether or not pretrained representations encode sufficient semantic information for equivalence decisions.}

We probe semantic equivalence from the pretrained models for pairs of \emph{unseen} sentences. We embed each sentence separately through the pretrained model, taking the last token representation for ALM and the average for MLM.\footnote{\rev{The lack of a next sentence prediction task (Fn.~\ref{fn:no-bert}) leads to no supervision for a [CLS] token.}} \citet{voita-etal-2019-bottom} and \citet{haviv-etal-2022-transformer} have shown that the positional information is diluted at the top transformer layers of MLMs, but it is crucial for the truth value in our language. We, therefore, take a weighted sum (a.k.a. scalar mix) of all layers for compensation for MLM.\footnote{Formally, $\mu$'s output contains all layer representations.} We also found that these 
simple
methods for sentence representations sometimes do not perform well. We hence additionally consider a variant where the probe is an attention-weighted mixture of all token positions.
We refer to these two representations as \noattnprobe and \attnprobe, respectively. See \S\ref{sec:training-details-pl} for more on their details.
We train a bilinear classifier probe on top of the sentence representations~\citep{li-etal-2021-implicit} and
evaluate it with accuracy on a held-out test set.
\rev{For each setting, we train the same probe with five different random seeds and report their mean and standard deviation.}
We report hyperparameters in \S\ref{sec:training-details-pl}.

\begin{table}[t!]
    \centering
    \begin{tabular}{@{\hspace{1pt}} l@{\hspace{4pt}}cccc @{\hspace{1pt}}}
        \toprule
        &
        \multicolumn{2}{c}{\textbf{ALM} {\footnotesize (à la GPT-2)}} & \multicolumn{2}{c}{\textbf{MLM} {\footnotesize (à la RoBERTa)}} \\
        \cmidrule(lr){2-3} \cmidrule(lr){4-5}
        & \textbf{Random} & \textbf{Trained} & \textbf{Random} & \textbf{Trained} \\
        \midrule
        & \multicolumn{4}{c}{\emph{Probing:} \noattnprobe} \\
        $L_t$ & 49.9\rev{$_{\pm 0.3}$} & \phantom{0}98.8\rev{$_{\pm 0.0\phantom{0}}$} & 50.0\rev{$_{\pm 0.4}$} & 50.1\rev{$_{\pm 0.2}$} \\
        $L_n$ & 50.0\rev{$_{\pm 0.3}$} & \phantom{0}79.9\rev{$_{\pm 0.2\phantom{0}}$} & 49.9\rev{$_{\pm 0.1}$} & 49.5\rev{$_{\pm 0.1}$} \\
        \midrule
        & \multicolumn{4}{c}{\emph{Probing:} \attnprobe} \\
        $L_t$ & 49.9\rev{$_{\pm 0.6}$} & 100.0\rev{$_{\pm 0.0\phantom{0}}$} & 50.0\rev{$_{\pm 0.4}$} & 63.8\rev{$_{\pm 1.7}$} \\
        $L_n$ & 50.1\rev{$_{\pm 0.4}$} & \phantom{0}82.5\rev{$_{\pm 20.9}$} & 50.2\rev{$_{\pm 0.2}$} & 49.7\rev{$_{\pm 0.3}$} \\
        \midrule
        & \multicolumn{4}{c}{\emph{Direct evaluation}} \\
        $L_t$ & 50.0 & 97.0$_{\pm 6.8\phantom{0}}$ & 50.0 & 95.4$_{\pm 4.7}$ \\
        $L_n$ & 50.0 & 91.1$_{\pm 19.9}$ & 50.0 & 50.4$_{\pm 0.8}$ \\
        \bottomrule
    \end{tabular}
    \caption{\label{tab:probing}
    Probing and direct evaluation accuracy (\%) on random and pretrained models with autoregressive and masked LMs on our propositional logic test set.
    \revv{We report the results with both our transparent language $L_t$ and the perturbed language $L_n$~(\S\ref{sec:non-transparency}).}
    Probing checks the equivalence of two sentences, while direct evaluation computes the value of one sentence.
    For probing, we test two ways to obtain sentence representations, \rev{reporting the mean and standard deviation across five probe training seeds.}
    For direct evaluation, we report the mean and standard deviation across our five templates.\footnotemark
    }
    \vspace{-10pt}
\end{table}

\footnotetext{\rev{ALM Trained \attnprobe $L_n$ has a degenerate seed that led to around 50\% accuracy, hence the large variance. It is possible that additional configuration-specific hyperparameter tuning, which we did not perform, could reduce this instability.}}

\rev{Past work has cast doubt on whether probes faithfully reflect the representation's encoding of the information of interest, or if they directly learn the task~\citep{hewitt-liang-2019-designing}.
This is an especially important issue here as our \attnprobe sentence representation injects additional trainable parameters compared to a simple (bi)linear classifier.
To answer this question in our setting, we follow previous studies~\citepia{conneau-etal-2018-cram,tenney2018what, wu-etal-2021-infusing} and}
train a randomly initialized and similarly frozen control model with the same architecture: if LMs emulate meaning similarly to  \citet{merrill-etal-2021-provable}'s algorithm, we would expect the pretrained model to yield higher probing accuracy than the random model.

\paragraph{Results.}
The $L_t$ rows in the top two sections of Table~\ref{tab:probing} summarize the results. With a simple sentence representation (\noattnprobe), the pretrained ALM achieves near-perfect probing accuracy for $L_t$, though MLM performs at chance level.
An attention-based sentence representation enables 63.8\% accuracy\footnote{With additional linear layers, it could go up to 83.4\%\rev{$_{\pm 2.0}$} while the random model still performs at chance level. We did not include this in Table~\ref{tab:probing} for consistency with other settings.} for MLM and improves ALM's performance to 100\%.
Importantly, in this variant, the random baselines still perform at chance level, demonstrating that the additional parameters do not lead to an overly powerful probe.
We discuss the accuracy differences between ALM and MLM in \S\ref{sec:discussion}.
These results demonstrate that pretraining enables meaning emulation, though the meaning representation can be more deeply encoded than what can be extracted with a (bi)linear probe.
We note that it is expected that the performance of pretrained models does not reach 100\%. While \citet{merrill-etal-2021-provable} showed its theoretical possibility, their setup assumes active learning with unlimited access to $\assert$ and allows the ``probe'' $\delta$ to be an arbitrarily powerful function, among other differences.

\begin{table}[t!]
    \centering
    \begin{tabular}{@{} ccc @{}}
        \toprule
        & \textbf{--Reflexivity} & \textbf{+Reflexivity} \\
        \midrule
        \multirow{2}{*}{\textbf{--Symmetry}} & \texttt{a=b} & \texttt{a=b}, \texttt{a=a}, \texttt{b=b} \\
        & 50.5\rev{$_{\pm 0.4}$} & 92.7\rev{$_{\pm 0.1}$} \\
        \multirow{2}{*}{\textbf{+Symmetry}} & \texttt{a=b}, \texttt{b=a} & \texttt{a=b}, \texttt{b=a}, \texttt{a=a}, \texttt{b=b} \\
        & 50.3\rev{$_{\pm 0.3}$} & 98.8\rev{$_{\pm 0.0}$} \\
        \bottomrule
    \end{tabular}
    \caption{\label{tab:grounding} ALM probing accuracy (\noattnprobe; \%) on our propositional logic test set with pretraining data with different properties, where \texttt{a}, \texttt{b} are expressions in $L_t$.
    \rev{We report the mean and standard deviation across five probe training seeds.}
    }
    \vspace{-10pt}
\end{table}

\paragraph{Grounding.}
We found that independently sampling pretraining sequences results in unsuccessful emulation with probing performance at random.
Instead, it is crucial to ground \texttt{=} %
with reflexivity and symmetry.\footnote{
Reflexivity states that $a=a$,
and symmetry $a=b \Rightarrow b=a$.
Equality further requires transitivity: $a=b \land b=c \Rightarrow a=c$,
but it is not tested in our probing setup and
we found it unimportant for probing accuracy in preliminary experiments.
}
We achieve this by augmenting the pretraining data: if \texttt{a=b} is a pretraining sequence,
we ensure \texttt{a=a}, \texttt{b=b} (reflexivity), and \texttt{b=a} (symmetry) are too.
This imposes a constraint on the pretraining data distribution that eases the learning of \texttt{=}'s meaning.
Table~\ref{tab:grounding} shows that both properties are important.
We consider the implication in \S\ref{sec:discussion}.

\subsection{Analysis: Direct Evaluation on $L_t$} \label{sec:direct-eval}

The process of training a probe introduces additional complexity, such as $\pm$\textsc{Attn}, that potentially complicates our analysis.
Therefore, we also test a stronger condition where there is no additional classifier: can the pretrained models \emph{evaluate} expressions, without any further training (e.g., a probe)?
For MLM, it is the most straightforward to compare if the model assigns a higher probability to \texttt{T} or \texttt{F} in \texttt{sentence=[MASK]}. However, this is a sequence that never occurs in the pretraining corpus since a standalone \texttt{T} or \texttt{F} is not part of our language (Eq.~\ref{eq:cfg}).
Therefore, we use five templates on the right-hand side that are minimal in our language: $($\texttt{T$\land$[MASK]}$)$, $($\texttt{F$\lor$[MASK]}$)$, $($\texttt{[MASK]$\land$T}$)$, $($\texttt{[MASK]$\lor$F}$)$, $($\texttt{$\lnot$[MASK]}$)$.
\rev{For the first four templates, we expect the masked position to be filled with the truth value of the proposition, and the negated value for the last one.}
For ALM, we compare if the model assigns a higher probability to the sequence where \texttt{[MASK]} is filled in with \texttt{T} vs. \texttt{F}.

\paragraph{Results.}
The bottom section of Table~\ref{tab:probing} shows the mean and standard deviation of the evaluation accuracy across our five templates. Without training, a random model always has 50.0\% accuracy on expectation.
Both ALM and MLM achieve a high evaluation accuracy, above 95\%, corroborating the LMs' capability to represent the meaning of $L_t$.

These results respond to the argument in \citet{bender-koller-2020-climbing}:
\begin{displayquote}
We let GPT-2 complete the simple arithmetic problem \emph{Three plus five equals}. The five responses below [...] show that this problem is beyond the current capability of GPT-2, and, we would argue, any pure LM.
\end{displayquote}
We showed that form-only supervision \emph{does} allow such evaluation on a strongly transparent language\rev{, at least when the supervising data distribution satisfies symmetry and reflexivity}.

\subsection{Non-transparency} \label{sec:non-transparency}
Building towards non-transparent natural language, it is important to understand strong transparency's effect on emulatability.
We design a minimally perturbed version of $L_t$ that is non-transparent, $L_n$.
The syntax stays the same, but we change the semantics such that $\lnot$ has a side effect: when followed by \texttt{T} or \texttt{F}, it inverts the meaning of these literals that occur in certain other environments. Specifically, each $(\lnot \texttt{T})$ node changes the meaning of all the literals \texttt{T} in its c-commanded subtree (i.e., the $e$ subtree headed by the $(\lnot \texttt{T})$ node's sibling, if there is one; \citealp{reinhart1976syntactic}) to \texttt{F}.
An additional $(\lnot \texttt{T})$ does not invert back.
Similarly, $(\lnot \texttt{F})$ changes the meaning of the literal \texttt{F} to \texttt{T}.
For example, in the sentence in Figure~\ref{fig:logic}, the \coloredbox{red}{\texttt{T}} node is c-commanded by (or, a descendant of a sibling of) the \coloreddashedbox{blue}{$e$}$\rightarrow (\lnot \texttt{T})$ node, so its meaning is changed to \texttt{F}.
On the other hand, the $e \rightarrow (\lnot$\coloredbox{red}{\texttt{T}}$)$ node does not invert the meaning of the unboxed \texttt{T} because they do not constitute a c-command relation.
This alternation is inspired by binding theory in generative grammar~\citep{chomsky1981lectures,chomsky1983some}, where the $(\lnot \texttt{T})$ node is the binder that c-commands the bindee.
Since the meaning of \texttt{T} and \texttt{F} now depends on the existence of a binder, $L_n$ is non-transparent.\footnote{This is a straightforward way to introduce a $\lnot$ with side effect to a hierarchical structure. An alternative is to rely on a linear structure and inverts all literals linearly following $\lnot$. Nevertheless, our version leverages the hierarchical reasoning that the model originally needs to possess to evaluate an expression, while this version requires a new type of reasoning that is linear. So that change would be less minimal.}

\paragraph{Results.}
We conduct the same pretraining/probing/direct evaluation procedure on $L_n$.
Table~\ref{tab:probing} reports the results.
Non-transparency decreases ALM's probing accuracy with both \noattnprobe and \attnprobe, though not to random level.
\rev{The variance across different probe training seeds also increases compared to $L_t$, indicating that the pretrained representation is less robust.}
Directly evaluating ALM with $L_n$ similarly leads to both decreased average accuracy and increased variance.
MLM, on the other hand, achieves random probing and evaluation accuracy.
Overall, the lack of strong transparency reduces models' meaning emulation ability, though not always to chance performance.

%% file: sections/04_nl.tex
\section{What About Natural Language?} \label{sec:nl}
While existing LM architectures and objectives are able to emulate the meaning of synthetic languages, it is unclear how these observations transfer to natural language (NL). \citet{merrill-etal-2021-provable} hinted that, since NL is non-transparent and likely more complex than their constructed non-emulatable language, it is probable that a pretraining procedure, even with $\assert$-access, cannot emulate its meaning either. This, however, remained an untested hypothesis.

We formalize this intuition and prove that a specific NL phenomenon, referential opacity, makes NL non-transparent.\footnote{Deictic expressions are another example, though they have been extensively studied under coreference resolution.} This phenomenon has been widely studied in semantics~\citepia{quine1956quantifiers,kripke1972naming}, yet it has received little attention in modern NLP. We fill this gap from the perspective of strong transparency and study the representation of this phenomenon in modern LMs with a probing-based and a sentence similarity-based analysis.

\subsection{Referential Opacity}

To illustrate referential opacity, we use the classic example in semantics:

\begin{example} \label{ex:superman}
\phantom{ }
\begin{exe} \vspace{-5pt}
    \exi{(a)} Lois Lane believes Superman is a hero.
    \exi{(b)} Lois Lane believes Clark Kent is a hero.
\end{exe}
\end{example}

Note that (a) and (b) have different truth conditions: their truth values differ if Lois Lane does not know Superman and Clark Kent are the same person.
Formally, $\dntc{Lois Lane believes Superman is a hero.}{\lambda^2}\ne\dntc{Lois Lane believes Clark Kent is a hero.}{\lambda^2}$.\footnote{In this section we consider the language $L$ to be English, or any NL that exhibits this phenomenon, and $\dnc{\cdot}{\cdot}$ to be intensions~(\S\ref{sec:meaning}). We drop the subscript $L$ for brevity.}
On the other hand, $\dntc{Superman}{\lambda^2}=\dntc{Clark Kent}{\lambda^2}$.\footnote{It is possible to argue that $\dntc{Superman}{\lambda^2}\ne\dntc{Clark Kent}{\lambda^2}$ if we consider their intension to be different. Nevertheless, we adopt the view of \citet[\S12.3]{kratzer1998semantics} to not introduce intensionality by default (i.e., with $\kappa=\lambda^2$), but rather it is evoked by context: ``The usual denotations are extensions. But for nonextensional contexts, Intensional Functional Application allows a switch to intensions. The switch is triggered by particular lexical items [...]''.}
In other words, two expressions that have the same denotation, when embedded in the same context, yield sentences with different truth conditions.
Such contexts are called \textbf{referentially opaque}, and, in this case, they are induced by a propositional attitude verb ``believes'' whose meaning depends on the cognitive state of its subject~\citep{propsitional_attitude}.

Now we formalize referential opacity:

\begin{definition}
In natural language, an expression $e$ is contextually valid in $\kappa = \dpt{l,r}$ if none of $\dnc{l}{\lambda,er}, \dnc{e}{l,r}, \dnc{r}{le,\lambda}$ is $\emptyset$.\footnote{This is a technical detail needed for proving Theorem~\ref{thm:our-lovely-theorem}.}
\end{definition}

\begin{definition} \label{def:opacity}
A context $\kappa = \dpt{l,r}$ in natural language is \term{referentially opaque} if there exist expressions $e_1, e_2$, both contextually valid in $\kappa$, such that $\dnc{e_1}{\lambda^2} = \dnc{e_2}{\lambda^2}$ and $\dnc{le_1r}{\lambda^2} \ne \dnc{le_2r}{\lambda^2}$.
\end{definition}

Def.~\ref{def:opacity} matches the linguistic phenomenon: let $e_1$=``Superman'', $e_2$=``Clark Kent'', and the opaque context $\kappa$=$\dpt{\text{``Lois Lane believes''},\text{``is a hero.''}}$, and we recover our analysis of Ex.~\ref{ex:superman} above.

Now, we prove that %
the existence of referentially opaque contexts implies non-transparency. 
We assume compositionality, for which we provide a working definition: $\dnc{ler}{\lambda^2}=f(\dnc{l}{\lambda,er}, \dnc{e}{l,r}, \dnc{r}{le,\lambda})$ for some meaning composition function $f$.\footnote{This is a mild assumption, considering the generality of compositionality~\citep{fodor1988} and that our definition is weak, e.g., weaker than \citet{andreas2018measuring}'s.
}
\rev{Intuitively, the proof shows that if all expressions have fixed meaning (i.e., are strongly transparent), referential opacity would not arise.}

\begin{theorem} \label{thm:our-lovely-theorem}
A compositional language with referentially opaque contexts is not strongly transparent.
\end{theorem}
\begin{proof}
Suppose by contradiction we have such a language $L$ that is strongly transparent. Let $e_1, e_2$ be expressions in some opaque context $\dpt{l,r}$ in $L$.
\begin{align*}
    \dnc{le_1r}{\lambda^2} &=& f\left(\dnc{l}{\lambda,e_1r}, \dnc{e_1}{l,r}, \dnc{r}{le_1,\lambda}\right) \\ && \text{By compositionality} \\
    &=& f\left(\dnc{l}{\lambda,e_2r}, \dnc{e_1}{\lambda^2}, \dnc{r}{le_2,\lambda}\right) \\ && \text{By strong transparency} \\
    &=& f\left(\dnc{l}{\lambda,e_2r}, \dnc{e_2}{\lambda^2}, \dnc{r}{le_2,\lambda}\right) \\ && \text{By referential opacity premise} \\
    &=& f\left(\dnc{l}{\lambda,e_2r}, \dnc{e_2}{l,r}, \dnc{r}{le_2,\lambda}\right) \\ && \text{By strong transparency} \\
    &=& \dnc{le_2r}{\lambda^2} \\ && \text{By compositionality}
\end{align*}
This violates $\dnc{le_1r}{\lambda^2} \neq \dnc{le_2r}{\lambda^2}$, the referential opacity premise.
So $L$ is not strongly transparent.
\end{proof}

Therefore, as a non-transparent example in NL, we study whether referential opacity is reflected in the representation of current LMs.

\subsection{Data} \label{sec:nl-data}

We cast referential opacity as a sentence pair binary classification problem.
We generate sentence pairs like Ex.~\ref{ex:superman} as our dataset.
Ex.~\ref{ex:superman} consists of two parts that correspond to the two conditions in Def.~\ref{def:opacity}: two co-referring expressions ($\dnc{e_1}{\lambda^2} = \dnc{e_2}{\lambda^2}$), and a referentially opaque context that embeds the entity ($\dnc{le_1r}{\lambda^2} \ne \dnc{le_2r}{\lambda^2}$).
Next, we separately introduce how we generate them.
Our final dataset consists of 45K/6K/6K training/development/testing sentence pairs for GPT-2 and 97K/12K/12K for BERT.
\S\ref{sec:nl-dataset-details} provides more details, including more fine-grained dataset statistics for different experimental settings below.

\paragraph{Co-referring expressions.}
The co-referring expressions in Ex.~\ref{ex:superman} are proper names, ``Superman'' and ``Clark Kent.''
Not only is this hard to collect data for, but, due to the rigidity of proper names~\citep{kripke1972naming}, it is also theoretically more challenging to analyze as the classic intensionality framework is more difficult to apply (\citealp{von2011intensional}).\footnote{Though see \citet{shabasson2018two} for a theorization.}
We hence consider co-referring expressions that are one proper name and one definite description, such as ``Yuri Gagarin'' and ``the first person in space,'' which can be more straightforwardly accounted for with intensionality~(\citealp[\S12]{kratzer1998semantics}; \citealp{von2011intensional}).
We use the LAMA dataset~\citep{petroni-etal-2019-language}, specifically the T-REx split~\citep{elsahar-etal-2018-rex} following recent factual probing work~\citep{jiang-etal-2020-know,shin-etal-2020-autoprompt,zhong-etal-2021-factual}, to obtain a list of such entities. To make sure the model representation captures the coreference,
we follow \citet{petroni-etal-2019-language} and use LAMA to prompt the LM with these equivalences and only keep entities that are correctly predicted.\footnote{Previous work~\citep{poerner-etal-2020-e,dufter-etal-2021-static,cao-etal-2021-knowledgeable} questioned whether such prompting measures model ``understanding.'' Our setup, though, does not depend on ``understanding'', but only requires association.}

\paragraph{Contexts.}
We construct referentially opaque and referentially transparent contexts to embed these co-referring expressions. We only consider referential opacity involving propositional attitude verbs, where the context is referentially opaque iff its main verb conveys propositional attitude. There are other types of referential opacity, such as counterfactuals~(\citealp{von2011intensional}; \citealp[\S7]{kearns2011semantics}) and substitutions that shift the syntactic status of constituents~\citep[e.g.,][]{fine1990}, that we omit in this work for simplicity, though they could be targets of future studies. We manually design two classes of templates, depending on the verb's argument structure. The first has an embedded clause, e.g.,
\begin{example}
Label = non-equivalent\footnote{
Consider, for example, if this person is Yuri's neighbor and wants to meet him for dinner, but, being an avid flat-earther, is not fond of space traveling and is unaware that he has been to space. She would say she wants to meet Yuri Gagarin but has no interest in meeting the first person in space.
}
\begin{exe} \vspace{-5pt}
    \exi{(a)} She wants to meet Yuri Gagarin.
    \exi{(b)} She wants to meet the first person in space.
\end{exe}
\end{example}
\noindent The second contains only the main clause, such as
\begin{example}
Label = equivalent
\begin{exe} \vspace{-5pt}
    \exi{(a)} He speaks Lao.
    \exi{(b)} He speaks the official language of Laos.
\end{exe}
\end{example}
\noindent The two sentences in a pair only differ by the entity reference: one is a name and one is a definite description.
A sentence pair is non-equivalent iff it has a referentially opaque context, or \rev{within our scope of study}, iff its main verb is a propositional attitude verb.
We gather the list of verbs from past linguistic studies and verify with native speaker judgment (see \S\ref{sec:nl-dataset-details}).

\subsection{Models}

We consider GPT-2-XL and BERT-large-cased\footnote{Not RoBERTa as in \S\ref{sec:lm}, because BERT's [CLS] token can act as and is commonly taken to be the sentence representation~\citepia{bert,karpukhin-etal-2020-dense}.}, the largest variants in these two families, as representative autoregressive and masked LMs. They have 1.5B and 340M parameters, respectively. We obtain sentence representations in the same way as in \S\ref{sec:lm}, except without attention-weighting and simply using the [CLS] embedding for BERT.

\subsection{Analysis: Probing} \label{sec:nl-probing}

\begin{table}[t!]
    \centering
    \begin{tabular}{ @{\hspace{4pt}} ll cc @{\hspace{4pt}} }
        \toprule
        && \textbf{GPT-2} & \textbf{BERT} \\
        \midrule
        \multirow{3}{*}{Simple} & 
        Equiv. & 100.0$_{\pm 0.00}$ & 100.0$_{\pm 0.00}$ \\
        & Non-equiv. & 100.0$_{\pm 0.00}$ & 100.0$_{\pm 0.00}$ \\
        & Overall & 100.0$_{\pm 0.00}$ & 100.0$_{\pm 0.00}$ \\
        \midrule
        \multirow{3}{*}{Coord.}  & Equiv. & \phantom{0}85.3$_{\pm 0.03}$ & \phantom{0}72.4$_{\pm 0.03}$ \\
        & Non-equiv. & \phantom{0}15.5$_{\pm 0.03}$ & \phantom{0}29.6$_{\pm 0.03}$ \\
        & Overall & \phantom{0}50.4$_{\pm 0.00}$ & \phantom{0}51.0$_{\pm 0.00}$ \\
        \bottomrule
    \end{tabular}
    \caption{\label{tab:nl-probe} Probing accuracy (\%) for referential opacity on GPT-2-XL and BERT-large-cased. We report the mean and standard deviation across 10 seeds. We consider two types of sentences, simple sentences without attractors and coordinated sentences with attractors. For each type, we show both the label-specific accuracy (Equivalent/Non-equivalent) and the overall accuracy.
    }
    \vspace{-10pt}
\end{table}

We use the same bilinear probe in \S\ref{sec:lm} as a binary classifier over sentence pairs, determining the equivalence, or the referential transparency, of each pair. However, because of the lexicalized nature of referential opacity, the probe could easily overfit and recognize not their equivalence but the existence of a propositional attitude verb.

To overcome this, we introduce attractors~\citepia{linzen-etal-2016-assessing,gulordava-etal-2018-colorless,pandia-ettinger-2021-sorting}.\footnote{\rev{Another option is to have disjoint training and testing verbs. This did not work in preliminary experiments because verbs that induce referential opacity are semantically closer, as they always convey propositional attitude. So the model could use this similarity in the word embedding space to extrapolate.}}
\rev{We always conjoin a clause with a propositional attitude verb and one with a non-attitude verb, disallowing the aforementioned heuristics.
The equivalence label now depends on if the entity alternation occurs under the non-attitude verb, which would result in an equivalent sentence pair, or the attitude verb, which would lead to non-equivalence.}
For example:
\begin{example} \label{ex:attractor-eq}
Label = equivalent
\begin{exe} \vspace{-4pt}
    \exi{(a)} He speaks Lao and she wants to meet Yuri Gagarin.
    \exi{(b)} He speaks the official language of Laos and she wants to meet Yuri Gagarin.
\end{exe}
\end{example}
\begin{example} \label{ex:attractor-neq}
Label = non-equivalent
\begin{exe} \vspace{-4pt}
    \exi{(a)} He speaks Lao and she wants to meet Yuri Gagarin.
    \exi{(b)} He speaks Lao and she wants to meet the first person in space.
\end{exe}
\end{example}
\noindent Despite both examples having the same verbs, the sentence pair in Ex.~\ref{ex:attractor-eq} is equivalent, but Ex.~\ref{ex:attractor-neq} is not.
We are not using attractors for out-of-domain evaluation; instead, the training and test sets are i.i.d., but we break down the test set performance by categories.

We train a probe on GPT-2-XL and BERT-large over 10 random seeds.
Details are in \S\ref{sec:nl-training-details}.
Table~\ref{tab:nl-probe} reports the results. As expected, both models overfit with the attractor-less simple sentences, achieving perfect accuracy. With attractors in coordinated sentences, however, both models obtain near-random performance overall. Because the training and test sets are i.i.d., this means that semantic equivalence based on referential opacity cannot be probed in our setup from these two models, suggesting an inadequate representation of this phenomenon.\footnote{\rev{There might still be other more complex heuristics, but even so, the probe still fails. Hence we do not need additional attractors to rule out all possible heuristics.}} Interestingly, both models tend to predict equivalence more than non-equivalence \rev{(more prominent with GPT-2 than BERT)}, likely due to the nuanced nature of this task: without training, a human would likely judge equivalence on referentially opaque sentence pairs too.\footnote{Though, with training, it is relatively straightforward to perform this task for a human, so it is reasonable to test the ability in LMs.} \revv{See \S\ref{sec:nl-finetuning} for a set of experiments that show that LMs can potentially learn to capture referential opacity with semantic supervision following pretraining.}

\subsection{Analysis: Sentence Similarity} \label{sec:nl-sim}

As in \S\ref{sec:direct-eval}, the simplicity of a training-free analysis can be desirable.
To this end, we directly measure the cosine similarity between the two sentence representations in a pair. \revv{While this semantic similarity would be high for both groups of sentences by our construction,} equivalent sentence pairs should have more similar representations than those that are not. While factors other than semantics, such as syntax, also affect sentence representations, we strictly control them in our synthetic data generation to be identical between referentially transparent and opaque sentences.
We do not consider attractor sentences (\S\ref{sec:nl-probing}) in this analysis.

For significance testing, we employ an exact permutation test~\citep{fisher1935design} and a bootstrap test~\citep{EfroTibs93} with 1,000 iterations, performed across verbs, where the test statistic is the difference between the averaged cosine similarity of the two groups. Both tests are two-sided with the null hypothesis being that the model representation does not distinguish between the two classes of verbs. For GPT-2-XL, the permutation test gives $p=0.64$ and bootstrap gives $p=0.66$, barring us from rejecting the null hypothesis.
For BERT-large, they give $p=0.45$ and $p=0.57$ respectively, where we again observe no significant difference between the two classes.
\rev{Nonetheless, we note that the inability to reject the null hypothesis does not entail it is true.}

\revv{\citet{reimers-gurevych-2019-sentence} noted that computing sentence pair cosine similarity using BERT's [CLS] token, as we did, does not correlate well with textual similarity benchmarks. \revv{This phenomenon is commonly attributed to the anisotropic nature of pretrained representations~\citep{ethayarajh-2019-contextual}.} This does not undermine the validity of our method, which instead relies on the correlation between the cosine similarity and \emph{the model's representation of semantic closeness}. We ensure this correlation by controlling for all factors other than semantics \revv{(syntax, lexical choices, entities, etc.)}. Nevertheless, we also postprocess BERT's [CLS] representation using BERT-flow~\citep{li-etal-2020-sentence} which has been shown to increase the correlation with textual similarity benchmarks. We obtain a similar result: bootstrap gives $p=0.49$. While the two-sided permutation test gives $p=0.03$ with \revv{potential} significance, the one-sided version gives $p=0.99$; in other words, the calibrated space represents opaque sentence pairs to be more similar than transparent ones, contrary to our expectation that equivalent sentence pairs should be closer in the representation space than non-equivalent ones when all other factors are controlled.}

The results from these two sets of analyses in \S\ref{sec:nl-probing} and \S\ref{sec:nl-sim} are consistent and show no evidence of modern LMs representing referential opacity, demonstrating that they cannot fully emulate the meaning of NL.
Our finding adds to recent observations that pretrained LMs do not \revv{represent semantic phenomena well}~\citepia{tenney2018what,kovaleva-etal-2019-revealing,wu-etal-2021-infusing}.
Theoretically, it also strengthens the connection between strong transparency and meaning emulatability with NL-based empirical evidence.

%% file: sections/05_discussion.tex
\section{Discussion} \label{sec:discussion}

Through analyses based on probing and direct evaluation, we have seen that existing LM architectures and objectives can learn to emulate the meaning of a strongly transparent language $L_t$ when the training data reflects equivalence relations. While non-transparency ($L_n$) causes this ability to decrease, the trained models still outperform a random model in certain setups. We believe this result hints at the strength of current LM architectures and objectives.\footnote{Especially since our setting is more challenging than \citet{merrill-etal-2021-provable}'s algorithm, without their unlimited $\assert$-access, active learning, arbitrarily powerful $\delta$, etc. Plus, we restrict $\assert$ queries to be sentences and disallow comparing a sentence with \texttt{T} or \texttt{F} using $\assert$.} There seems to be a limit to this strength, though---in natural language, neither GPT-2 nor BERT represents the non-transparent phenomenon of referential opacity well.

Our results shed light on the relationship between the strong transparency of a language and whether its semantics can be emulated. We observed co-variation between the two: when slightly perturbed to be non-transparent, our logic language becomes harder to emulate; and there is no evidence for LMs representing the semantics of a non-transparent NL phenomenon. Nevertheless, the above-random emulation performance with $L_n$ suggests that there could be language properties that potentially better predict emulatability, leaving room for future theoretical endeavors.

We also found that, with a similar size and training procedure~(\S\ref{sec:pretraining}), ALM is more suitable for representing the meaning of our propositional logic languages than MLM, in our setup.
ALM achieves better probing accuracy than MLM under both methods of obtaining sentence representations that we explored.
Also, MLM completely fails to emulate meaning facing non-transparency, but not ALM.
Ultimately, though, we hope to understand if this difference transfers to natural language.
Our NL investigation reveals that both ALM (GPT-2) and MLM (BERT) achieve chance-level probing performance on the one phenomenon that we inspected, likely due to its difficulty.
It would be interesting for future efforts to further examine their differences, if any, in learning and representing the meaning of other NL phenomena.

Our results also lead to the question: why can LMs achieve above-random results on $L_n$ but not referential opacity?
While it is entirely possible that the latter is simply more difficult than our synthetic non-transparency, there are other factors at play.
First of all, natural language is much more variable than our synthetic language: utterances can be untruthful (though they are in general governed by Gricean quality; \citealp{grice1975logic}), subjective (such as our earlier claim about Corgis' cuteness, \S\ref{sec:assert}),
intensional (see \citealp{merrill-etal-2021-provable} for a discussion), etc.
But putting these variations aside, we saw from \S\ref{sec:lm} that even the synthetic language requires an explicit grounding of \texttt{=} to enable emulation, and this is missing from NL pretraining. It is certainly not the case that, for every expression such as ``Corgis are the cutest dogs.'' that exists in the pretraining corpus, the variations ``The cutest dogs are Corgis.'', ``Corgis are Corgis.'', ``The cutest dogs are the cutest dogs.'' are also guaranteed to appear. So perhaps there needs to be a more foundational change in our pretraining objective. As \citet{gpt3} foretold, ``A more fundamental limitation of [...] scaling up any LM-like model [...] is that it may eventually run into (or could already be running into) the limits of the pretraining objective.'' Our results point to one such possibility: we believe research into a more explicit representation of semantic relations in future pretraining processes, such as based on paraphrases, could be fruitful.

What we did not investigate, though, is whether partial equivalence grounding enables emulation: what if, for example, only 1\% of the pretraining data has this form of grounding, while the rest does not? And the above format already exists for certain sentences in NL. This, too, could be an exciting future research question.

%% file: sections/06_i_dont_want_to_have_a_related_work_section_but_the_reviewer_wont_be_happy_without_it_so_i_have_to_yield.tex
\section{\rev{Related Work}}

\rev{\citet{bender-koller-2020-climbing} initiated the discussion on the possibility of a learner acquiring meaning from training on linguistic forms alone. From first principles, they argued for its impossibility. Empirically, \citet{traylor-etal-2021-mean} also found that LMs cannot well-represent lexical-level symbols when the pretraining data is distributionally constrained to supply relevant signals. \citet{merrill-etal-2021-provable}, on the other hand, proved theoretically that it is possible to emulate the meaning of strongly transparent languages with assertion oracle access. We showed in this work that, empirically, LMs also attain the capability. \citet{patel2022mapping} is also conceptually similar to our work, discovering that the internal representation of LMs is to a large extent isomorphic to the conceptual spaces of directions and colors. They adopted in-context learning~\citepia{gpt3} to elicit the isomorphism, while we used the more traditional probing paradigm.}

\rev{Another line of work has inspected the extent to which pretrained LMs encode various types of semantic information. Some have examined the representation of lexical semantics: \citet{gari-soler-apidianaki-2021-lets} found that BERT representations reflect polysemy levels, and \citet{vulic-etal-2020-probing} showed that they also capture abundant type-level lexical knowledge. On the other hand, \citet{ettinger-2020-bert} and \citet{ravichander-etal-2020-systematicity} have discovered that pretrained LMs do not satisfactorily encode negation and hypernymy, respectively. Moving beyond the lexical level, \citet{wu-etal-2021-infusing} demonstrated that pretrained BERT and RoBERTa models less readily surface semantic dependency information than syntactic dependencies, while \citet{li-etal-2021-implicit} identified evidence of dynamic semantics representation in these models.}

%% file: sections/07_now_that_we_have_a_related_work_section_we_have_to_have_a_conclusion_section_too_otherwise_it_would_be_unconventional.tex
\section{\rev{Conclusion}}

\rev{We have empirically shown that pretrained language models are able to emulate the meaning of a strongly transparent language through pretraining on an assertion-inspired format, but this ability deteriorates when the language is minimally perturbed to be no longer strongly transparent. Furthermore, we found no representation of referential opacity, which is significant for being a non-transparent natural language phenomenon, in pretrained LMs.}

%% file: sections/99_appendix.tex
\section{Propositional Logic Dataset Details} \label{sec:details-pl}

We hand-designed the PCFG probabilities in Eq.~\ref{eq:cfg}. To expand an $e$, the two binary rules each have 0.06 probability under $L_t$. The $\lnot$ rule and expansion to \texttt{T} and \texttt{F} divide the remaining probability mass, with \texttt{T} and \texttt{F} having the same probability, half of the $\lnot$ rule. As $S$ does not expand to \texttt{T} or \texttt{F}, the other three rules proportionally split the probability mass.
We consider each of $(, ), \land, \lor, \lnot, \texttt{T}, \texttt{F}, \texttt{=}$ as a separate token for tokenization.
We enforce a maximum length of 248 tokens.
We sample all sentences without replacement.
The average $L_t$ sentence length is $\approx$48.6 tokens.
Sampling $L_n$ results in slightly longer sentences, so we decrease the binary rule probabilities to be 0.03 each, but the specification is otherwise the same.
The resulting $L_n$ sentence on average has $\approx$51.7 tokens.
We sample 819.2M pretraining sentences and 1M/10K/10K probe training/validation/test sentences.
Then, for each split, we sample sentence \emph{pairs}, with the same number as the number of sentences in that split.

\section{Propositional Logic Training Details} \label{sec:training-details-pl}

For pretraining, we mostly follow the original hyperparameters for GPT-2-small and RoBERTa-base. We train with batches of 8,192 sequences for 100k steps, equivalent to 1 epoch over our pretraining data. We use the AdamW optimizer~\citep{loshchilov2018decoupled} with epsilon $10^{-8}$ for ALM for $10^{-6}$ for MLM, and $\beta_2$ 0.95 for ALM and 0.98 for MLM. We set the learning rate to $6\times 10^{-4}$ warmed up over 10k steps with a 0.1 weight decay.

For probing, \attnprobe trains a query vector that interacts with the key representation of each token, obtained with a trained key matrix transformation, and the resulting attention weights are used to average the token embeddings.
We train all probes for 3 epochs with batch size 8 and 1,000 warmup steps and select checkpoint with validation accuracy. We use AdamW with $10^{-5}$ learning rate except only for $L_n$ \noattnprobe ALM that benefits from a different learning rate $10^{-3}$. We clip gradients to unit norm.

\section{Referential Opacity Dataset Details} \label{sec:nl-dataset-details}

We detail the generation of our referential opacity dataset, separately discussing its two aspects~(\S\ref{sec:nl-data}).

\subsection{Generating Co-referring Expressions} \label{sec:nl-dataset-details-coreferring-expressions}

For fact probing on LAMA, we use the prompt in the form ``The official language of Laos is known as \_\_" which we found appropriate for the entity types in T-REx. If the LM correctly predicts ``Lao'', we consider this equivalence, or fact, captured by the model. As LAMA was designed to have 1-token answers with BERT's tokenization, we let BERT fill in the blank. This is not a guarantee for GPT-2's tokenization, so we run decoding for the same number of steps as the true answer's length with beam size 5 and no sampling. To further ensure that the predictions are reliable and not due to noise, we only keep entity categories with overall prediction accuracy $>$25$\%$. The resulting categories are ``P37 official language'', ``P364 original language of film or TV show'', ``P140 religion'', ``P103 native language'', and ``P36 capital''. This procedure results in 1,606 facts for GPT-2 and 2,962 facts for BERT.

\subsection{Generating Contexts}

We generate two types of contexts (\S\ref{sec:nl-data}). The first type contains an embedded clause, for which we construct templates for each entity category in \S\ref{sec:nl-dataset-details-coreferring-expressions}. For language entities, for example, one template is ``[\textsc{pronoun}] [\textsc{verb}] to speak [\textsc{entity}].'' A sentence pair is formed by filling in [\textsc{entity}] with a definite description vs. a proper name for a fact. We only consider the pronouns ``She'' and ``He'' in this work. We consider 6 referentially transparent verbs (``starts'', ``begins'', ``ceases'', ``stops'', ``managed'', ``failed'') and 6 referentially opaque verbs (``wants'', ``intends'', ``hopes'', ``begs'', ``preferred'', ``suggested'').  The second type of context contains only the main clause. We use the referentially opaque template ``[\textsc{pronoun}] dislikes [\textsc{entity}].'' and an entity category-specific referentially transparent template such as ``[\textsc{pronoun}] speaks [\textsc{entity}].'' In total, we have 64,672 sentence pairs for GPT-2 and 121,768 for BERT.

For our probing analysis, we also included attractors with coordinated sentences~(\S\ref{sec:nl-probing}). As there are a quadratic number of possible coordinations, we subsampled 59,548 such sentences for GPT-2 and 119,540 for BERT, similar to the number of attractor-less sentences. We split all sentence pairs 8/1/1 for training/validation/testing.

For our similarity analysis, for a cleaner significance test, we only consider sentence pairs with an embedded clause. This leaves 58,776 sentence pairs for GPT-2 and 111,312 for BERT.

\section{Referential Opacity Training Details} \label{sec:nl-training-details}

The probe is trained similarly to \S\ref{sec:training-details-pl} except for 1 epoch with batch size 256 and learning rate $10^{-5}$.

\section{\revv{Can Language Models Learn to Represent Referential Opacity With Appropriate Supervision?}} \label{sec:nl-finetuning}

\begin{figure}[t!]
    \centering
    \includegraphics[width=0.46\textwidth]{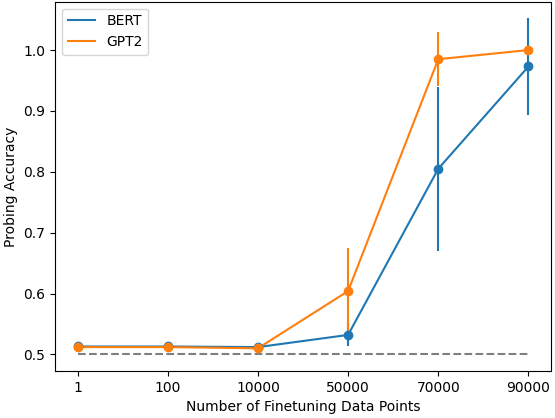}
    \caption{\revv{Probing accuracy after finetuning a pretrained LM on our (coordinated) referential opacity dataset with different numbers of finetuning examples. The mean and the standard deviation across 10 seeds are plotted. For clarity in visualizing the trend, \emph{the x-axis is not in linear scale.}}}
    \label{fig:finetuning-curve}
\end{figure}

\revv{
We showed in \S\ref{sec:nl} that we do not observe evidence of pretrained language models representing the phenomenon of referential opacity. A natural question, then, is whether language models can \emph{learn} to represent it. Following a similar setup as \citet{lyu-etal-2022-favorite} and \citet{liu-etal-2019-inoculation}, we finetune the entire model on a portion of our training set for 1 epoch and conduct the same probing procedure on the resulting model. All training is done with the coordinated data introduced (\S\ref{sec:nl-probing}). Finetuning uses the same hyperparameters in \S\ref{sec:nl-training-details}. Similar to \S\ref{sec:nl-probing}, we report the mean and standard deviation across 10 random seeds for each setting.
}

\revv{
We plot the probing accuracy along with the number of finetuning examples in Figure~\ref{fig:finetuning-curve}. Both GPT-2 and BERT continue to be unable to perform above-random with up to 10,000 finetuning examples, further demonstrating their inadequate semantic representation of referential opacity. Nevertheless, with enough finetuning examples, both models eventually achieve near-100\% probing accuracy. It is, therefore, possible that they can potentially learn to represent referential opacity with sufficient semantic supervision, though we note a caveat: while we introduced coordinated data to prevent an obvious shortcut that the model could take (\S\ref{sec:nl-probing}), it does not eliminate all possible shortcuts. It could be the case that the additional capacity afforded by finetuning enables the model to exploit a more sophisticated shortcut (unknown to us) instead of truly capturing this phenomenon.
}